\documentclass{article}

\usepackage{microtype}
\usepackage{graphicx}
\usepackage{subfigure}
\usepackage{booktabs} 

\usepackage{hyperref}

\usepackage{tikz-cd}

\usepackage[accepted]{icml2024_preprint}

\usepackage{amsmath}
\usepackage{amssymb}
\usepackage{mathtools}
\usepackage{amsthm}

\usepackage{booktabs} 
\usepackage{tikz}
\usepackage{amsfonts}

\newcommand\e{\mathrm{e}}
\newcommand\wt{\widetilde}
\newcommand\p{\partial}
\newcommand\Lcal{\mathcal{L}}

\usepackage[capitalize,noabbrev]{cleveref}

\theoremstyle{plain}
\newtheorem{theorem}{Theorem}[section]

\theoremstyle{definition}

\theoremstyle{remark}

\usepackage[textsize=tiny]{todonotes}

\icmltitlerunning{Unification of Symmetries Inside Neural Networks:
Transformer, Feedforward and Neural ODE}

\begin{document}

\twocolumn[
\icmltitle{ Unification of Symmetries Inside Neural Networks: \\
Transformer, Feedforward and Neural ODE}



\icmlsetsymbol{equal}{*}

\begin{icmlauthorlist}
\icmlauthor{Koji Hashimoto}{equal,kyoto}
\icmlauthor{Yuji Hirono}{equal,kyoto}
\icmlauthor{Akiyoshi Sannai}{equal,kyoto}
\end{icmlauthorlist}

\icmlaffiliation{kyoto}{ Department of Physics, Kyoto University, Kyoto, Japan }

\icmlcorrespondingauthor{Koji Hashimoto}{koji@scphys.kyoto-u.ac.jp}
\icmlcorrespondingauthor{Yuji Hirono}{yuji.hirono@gmail.com}
\icmlcorrespondingauthor{Akiyoshi Sannai}{sannai.akiyoshi.7z@kyoto-u.ac.jp}

\icmlkeywords{Machine Learning, ICML}

\vskip 0.3in
]



\printAffiliationsAndNotice{\icmlEqualContribution} 

\begin{abstract}
%
%
Understanding the inner workings of neural networks, including transformers, remains one of the most challenging puzzles in machine learning.
This study introduces a novel approach by applying the principles of gauge symmetries, a key concept in physics, to neural network architectures. 
By regarding model functions as physical observables, we find that parametric redundancies of various machine learning models can be interpreted as gauge symmetries.
We mathematically formulate the parametric redundancies in neural ODEs, and find that their gauge symmetries are given by spacetime diffeomorphisms, which play a fundamental role in Einstein's theory of gravity. 
Viewing neural ODEs as a continuum version of feedforward neural networks, we show that the parametric redundancies in feedforward neural networks are indeed lifted to diffeomorphisms in neural ODEs. 
We further extend our analysis to transformer models, finding natural correspondences with neural ODEs and their gauge symmetries. 
The concept of gauge symmetries sheds light on the complex behavior of deep learning models through physics and provides us with a unifying perspective for analyzing various machine learning architectures. 
\end{abstract}

\section{Introduction}

Neural networks have revolutionized various domains in recent years, from image recognition to natural language processing, demonstrating unprecedented accuracy and efficiency \cite{krizhevsky2012imagenet,silver2017mastering,brown2020language}. Despite these advancements, the intricate workings of neural networks, particularly transformer models, remain elusive, often perceived as `black boxes.' This lack of transparency hinders our ability to fully understand, trust, and optimize these models, posing significant challenges in both research and practical applications. This study aims to unravel these complexities, offering insights into the internal mechanisms of neural networks.


%
\begin{figure}[t]
  \centering
\includegraphics[width=1\linewidth]{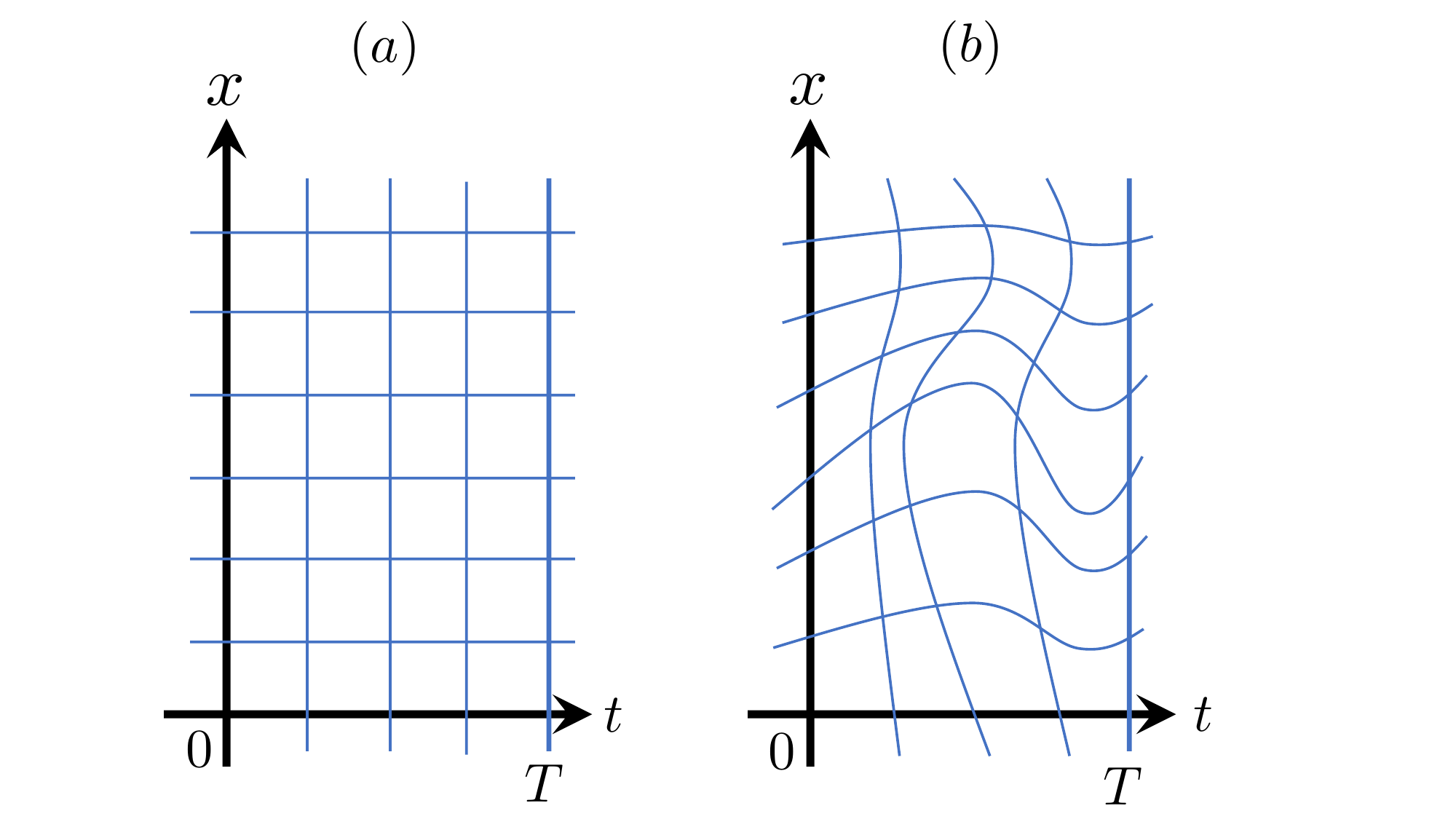}
  \vspace*{-8mm}
  \caption{Spacetime diffeomorphism used in proving Theorem \ref{thm:1}. (a) the original Cartesian spacetime grid. (b) A general coordinate transformation is performed on (a), with the boundary condition that at $t=0$ and at $t=T$ the transformation is trivial. }\label{fig:diffeo}
  \vspace*{-2mm}
\end{figure}
\begin{figure}[t]
\centering
\begin{tikzcd}[column sep=2cm, row sep=1.2cm]
{\rm NODE}_A 
\arrow[blue, leftrightarrow,swap]{d}{\text{diffeomorphism}}
\arrow[yshift=0.7ex]{r}{\text{discretize}}
& 
{\rm NN}_A 
\arrow[leftrightarrow]{d}{\text{param. rescaling}} 
\\
\color{blue}
{\rm NODE}_B
\color{black}
\arrow[blue,yshift=0.7ex]{r}{\text{discretize}}
& 
{\rm NN}_{B}
\end{tikzcd}
  \vspace*{-3mm}
\caption{Commutative diagram proven in Theorem \ref{thm:node-nn}. 
A pair of feedforward neural networks related by the weight rescaling symmetry
corresponds to a pair of neural ODEs related by the spacetime diffeomorphism.}
\vspace*{-2mm}
\label{fig:node-nn-diag}
\end{figure}
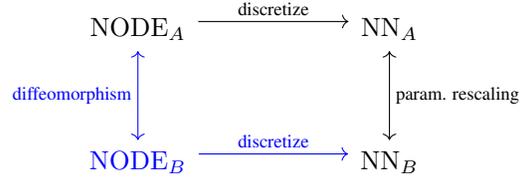

Building on the recognition of neural networks as complex systems akin to `black boxes,' we explore the potential of leveraging concepts from physics to demystify their inner workings. 
%
%
In physics, {\it gauge symmetries} represent a fundamental principle, encapsulating the invariance of physical observables under certain transformations of degrees of freedom.
Such gauge transformations are typically characterized by parameters depending on spacetime coordinates. 
This principle is exemplified by spacetime diffeomorphism invariance in general relativity \cite{Einstein:1916vd}, where physical laws remain consistent regardless of coordinate transformations. 
%
%
%
Regarding model functions of machine learning models 
as physical observables, parametric redundancies of machine learning models can be interpreted as 
gauge symmetries. 
In other words, gauge symmetries of neural networks 
refer to the transformations of weight/bias parameters 
that leave the model output intact.
%
One example of gauge symmetries is the rescaling symmetry
\cite{neyshabur2015path}\footnote{See also \cite{ioffe2015batch,nair2010rectified,badrinarayanan2015understanding} for a rescaling symmetry due to batch normalization, and \cite{kunin2020neural} for learning dynamics.} 
of feedforward neural networks, which refers to the scaling transformations of certain weights with which the input-output relationship is unchanged.
We can identify gauge symmetries of 
transformers \cite{vaswani2017attention},  
as a generalized type of parameter redundancy 
akin to the rescaling symmetry of feedforward neural networks.

%
We investigate the gauge symmetries 
of neural Ordinary Differential Equations (neural ODEs) \cite{chen2018neural}, which can be regarded as a continuous version of neural networks.
We show that the gauge symmetries of neural ODEs are mathematically characterized by ``spacetime''\footnote{The space and time correspond to the data space and depth of a neural network, respectively.}
diffeomorphisms (Theorem \ref{thm:1}, Fig.~\ref{fig:diffeo}, Sec.~\ref{sec3}). 
%
Then, we study the relation between the gauge symmetries of 
discrete feedforward neural networks and the neural ODEs.
%
Unlike the conventional approach of replacing derivatives with differences \cite{lu2018beyond,avelin2021neural}, our study introduces a new integrated relation where a time integration of the solution of neural ODEs directly corresponds to the inter-layer mapping of discrete neural networks. 
This relation is also effective for self-attention layers in transformers, linking them to nonlinear neural ODEs. 
By applying this integrated relation to feedforward neural networks and transformers, we reveal the correspondence between the rescaling symmetry and spacetime diffeomorphisms (Theorem \ref{thm:node-nn}, Fig.~\ref{fig:node-nn-diag}, Sec.~\ref{sec4} and Sec.~\ref{sec5}).

In this perspective, the internal symmetries of neural networks and transformers are unified as spacetime diffeomorphisms.
This in fact is consistent with the interpretation that
the neural network itself can be regarded as a spacetime.
There has been a growing research arena on the holographic principle \cite{Maldacena:1997re} for formulating a quantum gravity theory, where the gravitational spacetime itself is identified with a deep neural network
\cite{you2018machine,hashimoto2018deep}\footnote{See also subsequent development in \cite{hashimoto2018deep2,hashimoto2019ads,vasseur2019entanglement,tan2019deep,akutagawa2020deep,hu2020machine,yan2020deep,hashimoto2021neural,lam2021machine,song2021ads,yaraie2021physics,li2023learning}.}.  
Identifying the gauge symmetries of deep neural networks as spacetime diffeomorphisms is natural in the context, and this work builds a bridge connecting the symmetry of Einstein's theory of gravity with machine learning\footnote{As we will argue in Sec.~\ref{sec6}, the diffeomorphism in the linear neural ODE case acts on the weights and the biases as if they are a gauge field and a Higgs field in quantum field theories, respectively, which provides another bridge.}.


The main contributions of this work are summarized as follows.
\begin{itemize}
\item By regarding model functions as physical observables, we find that parametric redundancies of neural networks and neural ODEs are interpreted as gauge symmetries in physics.
\item We discover the mathematical characterization of the gauge symmetries in general neural ODEs (Theorem \ref{thm:1}, Fig.~\ref{fig:diffeo}), and they are given by spacetime diffeomorphisms.
\item We find that the parametric redundancy of feedforward neural networks is lifted to the spacetime diffeomorphism in neural ODEs, that can be seen as a continuum limit of feedforward neural networks (Theorem \ref{thm:node-nn}, Fig.~\ref{fig:node-nn-diag}).
\item We identify neural ODEs that naturally correspond to transformers, and show that their gauge symmetries can be analyzed in parallel with the cases mentioned above.
\end{itemize}



%


\section{Gauge symmetries in neural ODEs}
\label{sec3}

We here study parametric redundancies of generic neural ODEs. 
We will show that the redundancies in neural ODEs are 
characterized by spacetime diffeomorphisms with certain properties.

\subsection{General neural ODEs }

We consider a generic neural ODE~\cite{chen2018neural} given by 
\begin{equation}
\dot x^i (t) = F^i (t, x(t)),
\label{eq:eom-vec}
\end{equation}
where $x(t) \in \mathbb R^d$, 
the dot denotes the time derivative,
and $F^i$ is a force which can have $t$-dependent parameters.
We assume that $F^i$ is continuous in $t$ and globally Lipschitz continuous in $x$, that ensures the existence and uniqueness of the solution to an initial value problem.
We further assume that $F^i$ is smooth in $t$ and $x$.
For a given input $z \in \mathbb R^d$, the output of a neural ODE is given by 
$x(t=T)$, which is obtained by solving Eq.~\eqref{eq:eom-vec} 
with the initial condition $x(t=0) = z$. 
In the training of a neural ODE, the vector field on the RHS of Eq.~\eqref{eq:eom-vec} is optimized to minimize the loss function 
determined by the input-output relation (in the absence of other regularizers).
Generically, there can be different vector fields $F^i$ and $F'^i$ 
that give rise to the same input-output relation. 
Such neural ODEs are said to be equivalent. 
Indeed, there are infinitely many parametrizations of neural ODEs that are equivalent to each other. 
We will refer to such redundancies as {\it gauge symmetries}. 
In the following, we show that the gauge symmetry 
can be identified with the spacetime diffeomorphisms, that are fundamental gauge symmetries in Einstein's theory of gravity. 

%


\subsection{ Gauge symmetries in neural ODEs are spacetime diffeomorphisms }

To identify gauge symmetries in a neural ODE, it is convenient to formally extend the variables and forces to include the time variable as 
\begin{equation}
x^\mu (s)  
\coloneqq 
\begin{pmatrix}
    t(s) \\ x^i (s)
\end{pmatrix} ,
\quad
F^\mu (s, x(s))  
\coloneqq
\begin{pmatrix}
1
\\ 
F^i (s, x(s))
\end{pmatrix} .    
\end{equation}
The ODE~\eqref{eq:eom-vec} is equivalent to 
\begin{equation}
\frac{d}{ds} x^\mu (s)
 = 
 F^\mu (s, x(s)) .
 \label{eq:node-ex}
\end{equation}
The integration of the first equation results in $t(s) = s$. 
We slightly generalize the system to allow 
the zero-th component $F^0$ of $F^\mu$ to be different from $1$ 
while satisfying the condition that $t(0)=0$ and $t(T) = T$
and $t(s)$ is invertible. 

Let us prepare some terminologies to state our result.
Let $\Omega_0 \subset \mathbb R^d$ be the domain of a 
neural ODE (i.e., $x^i (0) \in \Omega_0$).
We define the support $\Omega$ of a neural ODE
to be the subspace of $\mathbb R^{d+1}$ in which
$x^\mu (s) \neq 0$ 
for 
some $x^\mu (0) \in \{0\} \times \Omega_0$
and 
for some $s \in I$.
When we denote the image of $\Omega_0$ by the time evolution to $T$ by $\Omega_T$ (i.e., $\Omega_T = \{x(T) \in \mathbb R^d \,\, | \,\, x(0) \in \Omega_0 \}$), 
the boundary of $\Omega$ is given by 
$\p \Omega = (\{0\} \times \Omega_0) \cup \overline{(\{T \} \times \Omega_T)}$, where the overline indicates the orientation reversal.
It is convenient to use the following notation for vector fields,
%
\begin{equation}
\hat x \coloneqq t \, \p_s + \sum_i x^i \p_i , 
\quad 
\hat F \coloneqq F^0 \p_s + \sum_i F^i \p_i ,
\end{equation}
We can express Eq.~\eqref{eq:node-ex} as
\begin{equation}    
\frac{d}{ds} \hat x(s) = \hat F (t,x) . 
\label{eq:node-vec-field}
\end{equation}

Observe that a trajectory $\{ x^\mu(s) \}_{s \in I}$
can be regarded as an image of $I$ in $\mathbb R^{d+1}$
by a map,
\begin{equation}
I \ni s \mapsto x (s) \in \mathbb R^{d+1}.
\end{equation}
This map is denoted as $\phi$. 
Under the current assumptions, the solution of Eq.~\eqref{eq:node-vec-field} is unique for an initial condition,
and a neural ODE can be seen as a collection of the images of $I$ in $\mathbb R^{d+1}$ that do not intersect with each other.
We define an infinitesimal deformation of a neural ODE generated by an infinitesimal vector field 
$\hat \epsilon = \sum_\mu \epsilon^\mu \p_\mu$
by the infinitesimal reparametrization of these trajectories along $\hat \epsilon$,
\begin{equation}
x^\mu (s) \mapsto x^\mu (s) +  \epsilon^\mu (x (s)) .
\end{equation}
Note that $\epsilon^\mu$ can depend on space-time variables $(t, x^i)$.
Such a transformation is called a general coordinate transformation in the theory of general relativity.
Let us identify the corresponding change of $\hat F$ under this reparametrization.
Substituting the reparametrized solution to Eq.~\eqref{eq:node-vec-field},
\begin{equation}
\frac{d}{ds} 
(\hat x + \hat \epsilon)
= 
\hat F (x + \epsilon) .
\end{equation}
By noting the relations
$\hat F (x + \epsilon) = \hat F(x) + \sum_\mu \epsilon^\mu \p_\mu \hat F(x)$ and $\p_s \hat \epsilon (x) =\sum_\mu \p_s x^\mu \p_\mu \hat\epsilon = \sum_\mu F^\mu \p_\mu \hat \epsilon$ to the first order in $\hat\epsilon$,
we can see that the new $\hat x$ should satisfy the following ODE,
%
\begin{equation}
\frac{d}{ds} \hat x (s)
= 
\hat F' 
\quad 
\text{with}
\quad 
\hat F'
\coloneqq 
\hat F 
+ 
\Lcal_{\hat \epsilon} \hat F,
\label{eq:node-deformation}
\end{equation}
where $\Lcal_{\hat \epsilon}$ is 
the Lie derivative along $\hat \epsilon$.
Namely, $\hat F'$ represents a new neural ODE obtained as a deformation of an original neural ODE given by
$\hat F$.

We here argue that that a certain class of spacetime diffeomorphisms
provide the essential characterization of gauge symmetries in neural ODEs.
Specifically, we show the following claim: 
\begin{theorem}
\label{thm:1}
Infinitesimal deformations
of a neural ODE~\eqref{eq:node-ex} 
that do not change the input-output relation
are in one-to-one correspondence with 
with infinitesimal diffeomorphisms 
$\Omega
\to \mathbb R^{d+1}$
that preserve the boundary $\p \Omega$ 
and the monotonicity of $t(s)$. 
\end{theorem}
\begin{proof}
The value $x^\mu (s=T)$ can be obtained 
by the integral of a one-form $dx^\mu$ 
over the image of $I$ in $\mathbb R^{d+1}$ by $\phi$,
\begin{eqnarray}
\lefteqn{
x^\mu (s=T )
- x^{\mu}(s=0)
}
\nonumber\\
&=&
\int_{\phi (I) } d x^\mu 
=
\int_I \phi^\ast ( d x^\mu )
= 
\int_I \frac{d x^\mu}{ds} ds ,
\end{eqnarray}
where $\phi^\ast$ is the pull-back of $\phi$.

Since we require that the input data are not transformed, 
the vector field should satisfy 
\begin{equation}
\epsilon^\mu (x (s=0)) = 0. 
\label{eq:epsilon-0}
\end{equation}
To keep the output $x^\mu (s=T)$ intact, 
the following integral should be 
invariant under the deformation along the vector field $\hat \epsilon$, which can be generated by the Lie derivative (see Eq.~\eqref{eq:node-deformation}),
\begin{equation*}
\begin{split}    
\delta \int_{\phi(I)} d x^\mu 
&= 
\int_{\phi(I)} \Lcal_{\hat \epsilon} d x^\mu  
= 
\int_{\phi(I)} d i_{\hat\epsilon} d x^\mu 
= 
\int_{\phi(I)} d \epsilon^\mu 
\\
&=
\epsilon^\mu (x (s = T))
- 
\epsilon^\mu (x (s = 0)).
\end{split}
\end{equation*}
Here, we used Cartan's magic formula, 
$\Lcal_{\hat\epsilon} = d i_{\hat\epsilon} + i_{\hat\epsilon} d$,
where $d$ is the exterior derivative
and $i_{\epsilon}$ denotes the interior product.
Combined with the condition~\eqref{eq:epsilon-0},
we find that the vector field $\epsilon$ should satisfy 
\begin{equation}
\epsilon^\mu (x (s = T)) = 0 .
\label{eq:epsilon-T}
\end{equation}
The conditions \eqref{eq:epsilon-0} and \eqref{eq:epsilon-T} are necessary so that the input-output relation of the model is unchanged. 
On the other hand, the transformation satisfying these conditions preserve the values of 
$x^\mu (s=0)$ and $x^\mu (s=T)$, which provides the sufficiency of these conditions to be a gauge symmetry.
Thus, we find that Eqs.~\eqref{eq:epsilon-0} and \eqref{eq:epsilon-T} 
provides the necessary and sufficient condition 
for the deformation generated by $\epsilon$ 
to be the gauge symmetry of the model. 
This concludes the proof. 
\end{proof}

Let us look at special cases more closely: a {\it spatial} diffeomorphism and {\it time reparametrization}, see Fig.~\ref{fig:ADM}.\footnote{In physics, a general coordinate transformation is decomposed into these, 
according to the ADM (Arnowitt-Deser-Misner) decomposition \cite{arnowitt1959dynamical}; the whole spacetime is made of equal-time slices each of which is parameterized by $x$.
Since neural ODEs treat the time $(t)$ direction special against the space (data $x$) direction, 
it is natural to consider these two special cases.}
The first case is a spatial diffeomorphism, which corresponds to the following reparametrization of the solution,
\begin{equation}
t \mapsto t, 
\quad 
x^i \mapsto 
x^i +  \epsilon^i (t, x),
\label{xi+epi}
\end{equation}
where $\epsilon$ is an infinitesimal parameter such that 
$\epsilon^i (t = 0, x) =
\epsilon^i (t = T, x) =0$.
Let us identify the new ODE. 
Since the combination $x^i + \epsilon^i$ satisfies 
Eq.~\eqref{eq:eom-vec}, 
\begin{equation}
\frac{d}{dt}    
\left( x^i + \epsilon^i \right)
= 
F^i (t, x + \epsilon). 
\end{equation}
To the first order in $\epsilon$, 
\begin{equation}
\dot x^i + 
\sum_j 
\epsilon^{i}_{\,\,,j}
\, \dot x^j 
+ \p_t \epsilon^i
= F^i + \sum_j  F^{i}_{\,\, ,j} \, \epsilon^j ,
\end{equation}
where ${(\cdot)}_{,j}$ denotes 
the derivative with respect to $x_j$,
$F^i{}_{,j} \coloneqq \frac{\p}{\p x^j} F^i$.
Thus, the new ODE is given by 
\begin{equation}    
\dot x^i
= 
F'^i 
= 
F^i + \sum_j F^{i}{}_{,j} \epsilon^j 
- \sum_j \epsilon^{i}_{\,\,,j} F^j - \p_t \epsilon^i. 
\end{equation}

As another special case, we look at time reparametrization,
\begin{equation}
t \mapsto t + \epsilon^0, 
\quad 
x^i \mapsto x^i. 
\end{equation}
When the component $\epsilon^0$ is nonzero,
the condition $t(s)=s$ is not satisfied after the transformation. However, we can always recover this condition by a time rescaling.
Let us take a space-independent time reparametrization for example.
The only nonzero component of the vector field generating this transformation is $\epsilon^0 (s)$ and it only depends on $s$.
This vector field modifies the ODE as 
\begin{align}
\p_s t(s) &= 1  - \p_s \epsilon^0 , \\
\p_s x^i (s) &= F^i + \p_s F^i \epsilon^0 . 
\end{align}
In this equation, $t(s) = s$ is not satisfied.
We can recover this relation by a time rescaling, 
and this rescaling can be done by writing $s$ derivative 
in terms of $t$ derivative, 
which results in 
\begin{equation}
\p_t x^i = F^i + F^i \dot{\epsilon}^0 + \p_t F^i \epsilon^0. 
\end{equation}
This is the modified ODE under a time reparametrization that is space-independent.

\begin{figure}[tb]
  \centering
 \includegraphics[width=1\linewidth]{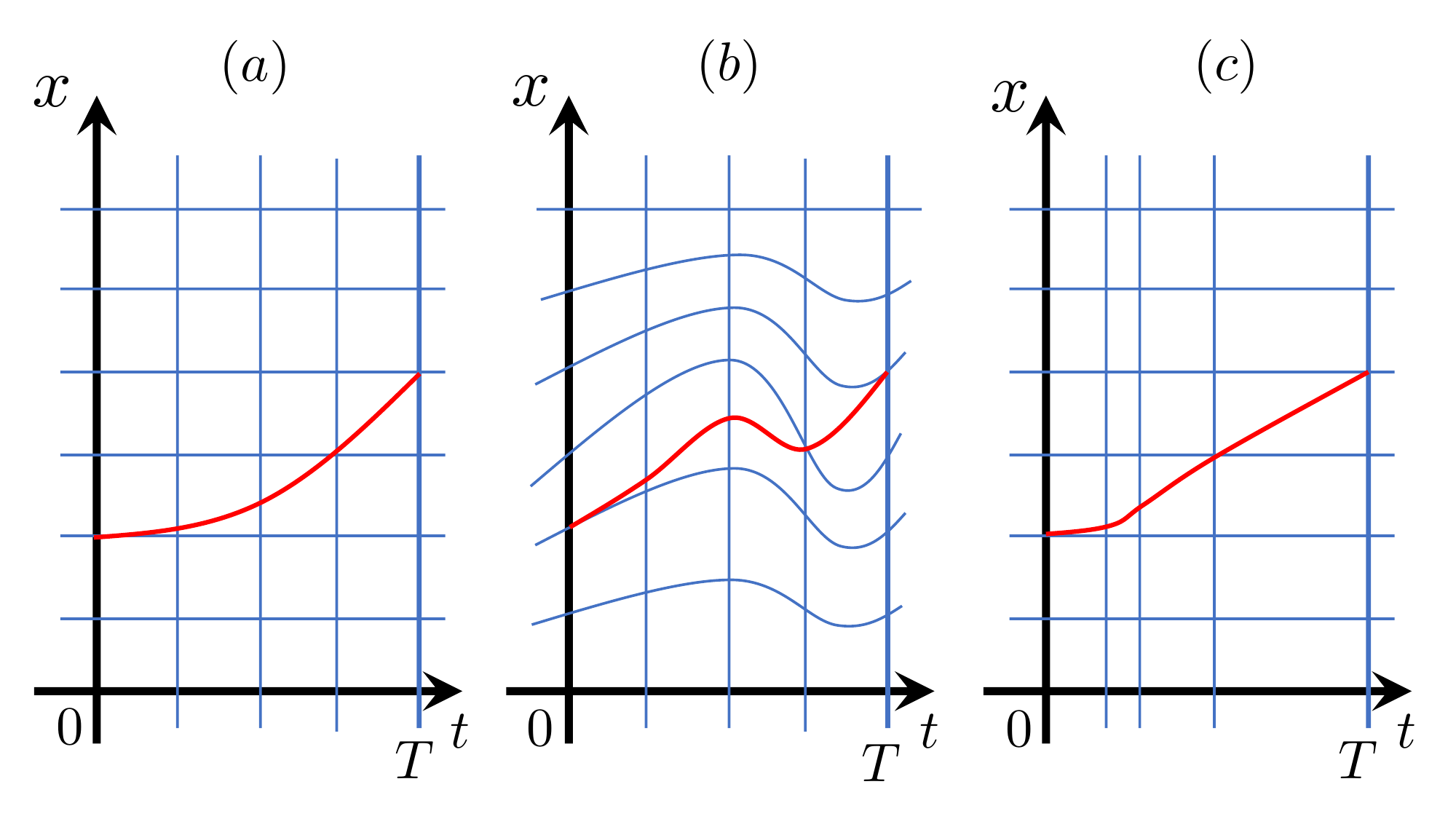}
 \vspace{-10mm}
   \caption{Spacetime diffeomorphism and their decompositions. Red line is the same trajectory $x(t)$. (a) the original spacetime grid. (b) A spatial diffeomorphism is performed on (a). (c) A time reparametrization is performed on (a).}
   \label{fig:ADM}
     \vspace*{-3mm}
\end{figure}

\subsection{Linear neural ODEs }

In order to make a connection to discrete neural networks, let us study gauge symmetries in linear neural ODEs in more detail.
We consider a multivariate linear ODE given by
\begin{equation}
\dot x^i(t)  = \sum_j w^i_{\,\, j}(t) x^j(t) + b^i(t),
\label{eq:dot-x-wx-b2}
\end{equation}
and hereafter we use the matrix notation and omit the indices $i$ and $j$ unless otherwise specified.
The mapping $\mathbb R^{d} \ni x(t=0) \mapsto x(t=T) \in \mathbb R^d$ gives 
the input and output of the neural ODE.
There can be difference choices of $\{w(t), b(t)\}_{t \in [0,T]}$ 
with the same input-output relation. 
Namely, there is a gauge symmetry, a redundancy in the choice of parameters. 
%
The explicit solution of this ODE is given by 
\begin{equation}
x(t=T)   
= W(w)_{T:0} 
\left(
x(0)
+ 
\int_0^T W(-w)_{0:t'} \, b(t') d t'
\right) .
\label{eq:lin-node-sol}
\end{equation}
Here, we have defined
\begin{equation}
W(w)_{t_1:t_2}
\coloneqq 
\begin{cases}
{\rm T}\, \e^{\int^{t_1}_{t_2} w(t') d{t'}}, &  t_1 \ge t_2 \\ 
\bar {\rm T}\, \e^{\int^{t_1}_{t_2} w(t') d{t'}}, & t_1 < t_2 
\end{cases}
\label{eq:wilson-line}
\end{equation}
where ${\rm T}$ and $\bar{\rm T}$
denote time ordering and reverse-time ordering, respectively. 
%


The gauge symmetry of this system can be found by 
reparametrizing the solution as 
%
\begin{equation}
x(t) \mapsto \wt{x}(t) = G(t) x(t) + c(t),
\label{eq:reparam-lin-node}
\end{equation}
where $G(T) = G(0) = 1$ and $c(0)=c(T)=0$.
While we discussed infinitesimal transformation in the previous section, here we consider finite transformations, that can be obtained by repeating infinitesimal ones.
Under this diffeomorphism,
the new $x(t)$ has the same initial and final values 
as the original $x(t)$. 
After the replacement, 
while $\wt{x}(t)$ satisfies the original ODE, 
the ODE governing the time evolution of $x(t)$ is now different. 
A direct substitution shows that the ODE 
governing the time evolution of new $x(t)$ 
can be obtained by the following transformation of weights:
\begin{equation}
\begin{split}    
w(t) &\mapsto G^{-1}(t) w(t) G(t) - G^{-1}(t) \dot G (t), 
\\
b(t) &\mapsto 
G^{-1}(t) \left[ b(t) + w(t) c(t) - \dot c(t) \right] .
\end{split}
\label{eq:w-b-transfm}
\end{equation}
One can check that 
the transformation
\eqref{eq:w-b-transfm} keeps 
$x(t=T)$ intact, using the explicit form
\eqref{eq:lin-node-sol}. 
Thus it is a gauge symmetry of the linear neural ODE.

An amusing similarity to fundamental physics is
that the transformation
by $G(t)$ on $w(t)$ has precisely the same form as that of
the gauge transformation in non-Abelian gauge theories. 
In addition, Eq.~\eqref{eq:wilson-line} is an analogue of open Wilson lines in non-Abelian gauge theories, regarding the
weight as a non-Abelian gauge field and the data $x$ or the bias $b$ as a quark field or a Higgs field in the fundamental representation.
This physical analogy will be used for discussing regularizations in Sec.~\ref{sec6}.

Note that the reparametrization~\eqref{eq:reparam-lin-node} includes infinitesimal time reparametrizations as special cases,
\begin{equation}
t \mapsto t' = t + \epsilon(t),
\end{equation}
where $\epsilon(t=0) = \epsilon(t=T) = 0$. 
An infinitesimal time reparametrization is realized 
through the following substitution,
\begin{equation}
\begin{split}    
x(t) 
&\mapsto 
x(t + \epsilon) 
\\
&= 
x(t) + \dot x(t) \epsilon (t) + O(\epsilon^2)
\\
&= 
\left[ 1 + \epsilon (t) w(t) \right]
x(t) + \epsilon(t) b(t) + O(\epsilon^2) ,
\end{split}
\end{equation}
where $\epsilon(t=0) = \epsilon(t=T) = 0$. 
Thus, the time reparametrization 
corresponds to the choice 
$G(t) = 1 + \epsilon(t) w(t)$
and 
$c(t) = \epsilon(t) b(t)$.

\section{Rescaling symmetry as a spacetime diffeomorphism}
\label{sec4}

We here study the connection between the rescaling symmetry  of deep feedforward neural networks and the spacetime diffeomorphism.
A neural ODE can be thought of as a continuum version of a deep feedforward neural network.
Through this correspondence, the rescaling symmetries of discrete neural networks are ``lifted" to the gauge symmetries characterized by the spacetime diffeomorphisms in the neural-ODE description.

\subsection{Review of the rescaling symmetry in feedforward neural networks}

Here we describe the rescaling symmetry in feedforward neural networks.\footnote{For convolutional neural networks, see App.~\ref{appB}.}
Consider a feedforward deep neural network 
\begin{align}
    x^{(n+1)}_i = \varphi\left(\sum_j w^{(n)}_{ij}x^{(n)}_j + b^{(n)}_i\right)
    \label{dnndef}
\end{align}
where $n$ is the index for the layers in the depth direction, and $i$ is the label of the unit in the layer. We suppose that every layer shares the same number of units. We take the activation function $\varphi(x) = {\rm ReLU}(x)$.\footnote{The leaky ReLU could also be used. See \cite{godfrey2022symmetries} for a classification of the symmetry for various activations.}
When the network has successive layers, the successive propagation in the network is 
\begin{align}
    &x^{(n+2)}_i 
    \nonumber \\
    &= \varphi\left(
    \sum_j w^{(n+1)}_{ij}
    \varphi\left(\sum_k w^{(n)}_{jk}x^{(n)}_k + b^{(n)}_j\right)+ b^{(n+1)}_i
    \right).
    \nonumber 
\end{align}
To describe the rescaling symmetry, let us consider the input-output relation for these successive layers. The input data is $x^{(n)}_k$ for all $k$. The output data is $x^{(n+2)}_i$ for all $i$. 
For the network output $x^{(n+2)}_i$ to be intact for any $i$ as a function of the input $\{x^{(n)}_k\}$, we are allowed to perform the following transformation on the weights and the biases, 
\begin{align}
\begin{split}
    w^{(n)}_{jk}  \mapsto \alpha^{(n+1)}_j w^{(n)}_{jk}, \quad
    b^{(n)}_{j}  \mapsto \alpha^{(n+1)}_j b^{(n)}_{j},
    \\
    w^{(n+1)}_{ij} \mapsto (\alpha^{(n+1)}_j)^{-1} w^{(n+1)}_{ij}, \quad
    b^{(n+1)}_{j}  \mapsto b^{(n+1)}_{j} 
\end{split}
\label{gsdnn}
\end{align}
for all $i$ and $k$, with any fixed $j$ and $n$. This keeps the output $x^{(n+2)}_i$ intact thanks to the scaling property of ReLU,
\begin{align}
 {\rm ReLU}(\alpha x) = \alpha \, {\rm ReLU}(x)
 \label{scalerelu}
\end{align}
for any $\alpha\in {\mathbb R}_{>0}$ and $x \in {\mathbb R}$. Thus
the network output is invariant for any 
$\alpha^{(n+1)}_j \in \mathbb{R}_{>0}$, this is a gauge symmetry of the feedforward neural network.\footnote{
The transformations \eqref{gsdnn} are local, meaning that they only affect the weights connected to the unit labeled by $(n+1,j)$.
This concept is similar to gauge symmetries in physics, where transformations are dependent on spacetime coordinates and are considered local for this reason.
From this perspective, the rescaling symmetry in neural networks can be naturally understood as a gauge symmetry.
}

\subsection{Diffeomorphism interpretation}

We here demonstrate that the rescaling symmetries 
in deep feedforward neural networks are promoted to spacetime diffeomorphisms
in their continuous counterpart, neural ODEs.
We show this connection explicitly using linear neural networks, which can be seen as a discrete version 
of the linear neural ODE~\eqref{eq:reparam-lin-node}. 
Here, rather than the popularly known discretization relation by replacing derivatives with differences \cite{lu2018beyond}, we use an integrated relation between the linear neural network and the linear neural ODE.
The output of each layer is given by 
\begin{equation}
x(n+1)
= 
\bar w_n x(n) + \bar b_n ,
\label{eq:linear-nn-rec}
\end{equation}
where $n = 0, \ldots, N-1$,
and 
the parameters $\bar w_n$ and $\bar b_n$ are expressed 
in terms those of the neural ODE as 
\begin{align}
&\bar w_n = {\rm T} \, \e^{\int_{n\Delta}^{(n+1)\Delta} w(t') d t' },
\label{eq:params-nn-node}
\\
&
\bar{b}_n = 
{\rm T} \, \e^{\int^{(n+1)\Delta}_{n\Delta} w(t') d{t'}}\!\!
\int_{n\Delta}^{(n+1)\Delta }\!\!\!\!
{\rm \bar T}\e^{ - \int_{n\Delta}^{t'} w(t'') d t''}
b(t') d t' .
\nonumber
\end{align}

In this linear neural network, the input-output relation is invariant under the 
following transformation of parameters,
\begin{equation}
\begin{split}    
\bar w_n &\mapsto (G_{n+1})^{-1} \bar w_n G_n,
\\
\bar b_n &\mapsto (G_{n+1})^{-1}
\left[ \bar b_n + \bar w_n c_n - c_{n+1} \right] ,
\end{split}
\label{eq:lin-nn-gauge}
\end{equation}
where 
$G_n$ are $d \times d$ invertible matrices with 
$G_{N} = G_{0} = 1$, 
and 
$c_n \in \mathbb R^d$ with $c_{N} = c_{0} = 0$.
One can check the invariance of $x(N)$ under 
the transformations~\eqref{eq:lin-nn-gauge}
by using the explicit solution of Eq.~\eqref{eq:linear-nn-rec},
\begin{equation}
\begin{split}    
x(N) 
&= 
\left(
\prod_{n=0}^{N-1} \bar w_n 
\right) 
\\
& 
\times 
\left(
x(0) + \sum_{m=0}^{N-1} 
(\bar w_0)^{-1} \cdots 
(\bar w_m)^{-1} \bar b_m
\right) .
\end{split}
\end{equation}

When two neural networks ${\rm NN}_A$ and ${\rm NN}_B$ are related 
by the transformation~\eqref{eq:lin-nn-gauge}, 
we write ${\rm NN}_A \sim {\rm NN}_B$
and say that they are equivalent. 
Similarly, when two neural ODEs ${\rm NODE}_A$ and ${\rm NODE}_B$ 
have the same input-output relation, we write 
${\rm NODE}_A \sim {\rm NODE}_B$
and say that they are equivalent. 
These are equivalence relations.

We argue that the redundancies 
\eqref{eq:lin-nn-gauge}
in linear neural networks are inherited from the gauge symmetries given by spacetime diffeomorphisms in the neural-ODE version of it.
Specifically, we have the following statement:
\begin{theorem} 
\label{thm:node-nn}
Let ${\rm NN}_A$ be a linear neural network that is obtained by the discretization of a neural ODE ${\rm NODE}_A$
and 
Let ${\rm NN}_B$ be a linear neural network 
equivalent to ${\rm NN}_A$ (i.e., 
${\rm NN}_B \sim {\rm NN}_A$).
%
Then, there exists a neural ODE ${\rm NODE}_B$
equivalent to ${\rm NODE}_A$ whose discretization gives ${\rm NN}_B$.
\end{theorem}

Intuitively, the theorem states that 
there is a spacetime diffeomorphism behind
a redundancy in discrete linear neural networks (see Fig.~\ref{fig:node-nn-diag} for schematic illustration). 
We provide a proof of this theorem in Appendix~\ref{app:proof-linear-nn}.

{\it Sketch of proof.} 
We make an explicit connection between
the redundancies~\eqref{eq:lin-nn-gauge} 
and the gauge symmetries in the neural ODE.
by matching the transformation 
parameters of the continuous and discrete models as 
$G_n = G(n \Delta)$
and 
$c_n = c(n \Delta)$. 

While the theorem above applies to the linear neural network with no activation, 
it is easy to generalize the statement to generic feedforward neural network with
ReLU activation; the only difference is the replacement of the transformation
parameter $G_n$ with the arbitrary diagonal matrix $\alpha^{(n)}$ whose diagonal entries are given by $\{\alpha^{(n)}_j\}$, as seen in \eqref{gsdnn}.\footnote{
We make a comment on generalities of our approach.
The redundancy characterized by our diffeomorphisms is with continuous parameters
and deformable to a trivial transformation smoothly,
while feedforward neural networks can have discrete redundancies such as permutations of latent nodes \cite{brea2019weight} (see \cite{wei2008dynamics,cousseau2008dynamics,amari2017dynamics,simsek2021geometry,entezari2021role} for special structure of weight space due to the permutation symmetry). 
In neural ODEs, the latter corresponds to exchanging two variables at specific points in time, which can be seen as a discrete gauge symmetry.
Recently in physics, discrete gauge symmetries have been studied extensively, and 
%
%
%
it would be interesting to generalize our results to such discrete symmetries.
}

\section{Transformers}
\label{sec5}

In this section, based on our integrated relation between the transformers \cite{vaswani2017attention} and the neural ODEs, we
find that a rescaling symmetry in a self-attention layer
can be understood as a spatial diffeomorphism.

\subsection{Rescaling symmetry in transformers}

Let us examine possible rescaling symmetries in a self-attention layer in transformers. In general, the self-attention layer is given as
\begin{align}
    h_I = \sum_{J=1}^n\varphi\left(
    (x_I W^{\rm (q)}) (x_J W^{\rm (k)})^T
    \right) (x_J W^{\rm (v)})
    \label{sa}
\end{align}
where $x \in {\mathbb R}^{n\times d}$ is a set of the data with $x_I \in {\mathbb R}^{d}$ $(I=1,2,\cdots, n)$ being the individual data, and $W^{\rm (q)}, W^{\rm (k)}, W^{\rm (v)}$ are the query, the key and the value weight, respectively, each of which takes a value in ${\mathbb R}^{d\times d}$. 

Normally in self-attention layers in transformers one takes 
softmax for the activation function $\varphi$, then there is no rescaling symmetry. 
Here, as we have seen in this paper, we use ReLU as $\varphi$ to find the rescaling symmetry.
In fact, there have been proposals to use ReLU for transformers \cite{choromanski2020rethinking,zhang2021sparse,qin2022cosformer,shen2023study,mirzadeh2023relu}.

The self-attention layer \eqref{sa} actually has the following rescaling symmetry. 
For any 
$A, B \in $ GL($d,{\mathbb R}$) and $\alpha \in {\mathbb R}_{>0}$ which satisfy
\begin{align}
    \sum_k A_i{}^k B_k{}^j = \alpha \delta_i{}^j,
\end{align}
the symmetry transformation is 
\begin{align}
    \begin{split}
    & W^{\rm (q)}{}_i{}^j \mapsto  \sum_k W^{\rm (q)}{}_i{}^k A{}_k{}^j,
    \\
    & W^{\rm (k)}{}_i{}^j \mapsto  \sum_k W^{\rm (k)} {}_i{}^kB^{T}{}_k{}^j,
    \\
    & W^{\rm (v)}{}_i{}^j \mapsto \alpha^{-1} W^{\rm (v)}{}_i{}^j.
    \end{split}
    \label{sagauge}
\end{align}
This is the gauge symmetry of the transformer.\footnote{
One can actually use this rescaling symmetry to gauge-fix some of the weights. For example, one can define a new weight matrix 
$    W^{\rm (qk)}{}_i{}^j := \sum_k
    W^{\rm (q)}{}_i{}^k [W^{\rm (k)}{}]^{T}{}_k{}^j $
which is invariant under the non-Abelian part of the rescaling symmetry. This roughly states that either the query weight or the key weight is unnecessary.}


\subsection{Non-linear neural ODE as a self-attention}
\label{subsec:node=sa}

In this subsection we identify the self-attention \eqref{sa} with an integrated nonlinear neural ODE, and find that a part of the rescaling symmetry \eqref{sagauge} is a spatial diffeomorphism.

%

Following \cite{lu2019understanding} (see also \cite{krotov2020large,zhong2022neural}), transformers can be understood as a discretized neural ODE, through the general relation of the discretization \cite{lu2018beyond}. What we propose here to relate the gauge symmetries is not the discretization, but an integration.
%
We consider a nonlinear correction to the linear neural ODE,
\begin{align}
    \dot{x}^i(t)= \sum_j w^i{}_j(t) x^j(t) + \sum_{j,k,l}\lambda^i{}_{jkl}(t) x^j(t)x^k(t)x^l(t)
    \label{xdxxx}
\end{align}
with the new parameter $\lambda$ treated perturbatively, $\lambda \ll 1$. To the leading order in $\lambda$, the integrated solution is
\begin{align}
    x^i(t) = \sum_j[W_{t:0}]^i{}_j x^j(0)
    + \sum_{j,k,l}\Lambda^i{}_{jkl} \, x^j(0)x^k(0)x^l(0)
    \nonumber
\end{align}
where $W$ is defined in Eq.~\eqref{eq:wilson-line} and 
\begin{align}
    \Lambda^i{}_{jkl}(t)&
    :=
    \sum_{i',i'',j',k',l'}
    [W_{t:0}]^i{}_{i'}\int_0^t\! dt'
    \biggl[
    [W_{0:t'}]^{i'}{}_{i''}
    \nonumber \\
    &
    \times
    \lambda^{i''}{}_{j'k'l'}(t') 
    [W_{t':0}]^{j'}{}_{j}
    [W_{t':0}]^{k'}{}_{k}
    [W_{t':0}]^{l'}{}_{l}
    \biggr].
    \nonumber
\end{align}
Now, for the fixed time period $t\in [0,T]$ integrated, consider a special case $[W_{T:0}]^i{}_j = \delta^i{}_j$. 
Then the solution of the nonlinear neural ODE \eqref{xdxxx} is
\begin{align}
    x^i(t) = x^i(0)
    + \sum_{j,k,l}\Lambda^i{}_{jkl} \, x^j(0)x^k(0)x^l(0).
    \label{solxxx2}
\end{align}
On the right hand side, the second term is a general form of the self-attention \eqref{sa} without the activation function, and the
first term is the additive term in the transformer.

In fact, it is possible to show that the transformer without the activation term is expressed as an integration of the neural ODE \eqref{xdxxx} as follows. 
We take an instantaneous nonlinearity,
\begin{align}
    \lambda^i{}_{jkl}(t) = \delta(t-t_0)\lambda^i{}_j \tilde{\lambda}_{kl}
    \label{tensorlambda}
\end{align}
with a tensor product of constants $\lambda^i{}_{j}$ and $\tilde{\lambda}_{kl}$, for $t_0\in (0,T)$.
Then, with the following identification
\begin{align}
\begin{split}
   & W^{\rm (q)}{}_{jm}=[W_{t_0:0}]^{m}{}_{j},\\
   &  W^{\rm (k)}{}_{km}=\sum_{k'}\tilde{\lambda}_{mk'}[W_{t_0:0}]^{k'}{}_{k},\\
   & W^{\rm (v)}{}_{li}=\sum_{i',j'}[W_{0:t_0}]^{i}{}_{i'}\lambda^{i'}{}_{j'}
    [W_{t_0:0}]^{j'}{}_{l},
\end{split}
\label{TrNODE}
\end{align}
we find that Eq.~\eqref{solxxx2} is rewritten as a transformer for $n=1$,
\begin{align}
    &x^i(T) = x^i(0)
    \nonumber \\
&    + \sum_{j,k,m,l}\!\!
\Bigl(x^j(0)W^{\rm (q)}{}_{jm} \Bigr)
   \Bigl(x^k(0)W^{\rm (k)}{}_{km}\Bigr)
   \Bigl(x^l(0)W^{\rm (v)}{}_{li}\Bigr).
   \nonumber
\end{align}
A generalization to the general token-length case $n> 1$ is straightforward, with the structural decomposition of the index $i$ to the pair $(i,I)$, and adopting the structure
\begin{align}
    \lambda^{iI}{}_{jJkKlL}
    = \lambda^i{}_j \tilde{\lambda}_{kl}
    \delta^I{}_K \delta_{JL}.
\end{align}
Thus, we conclude that the transformer without the activation can be viewed as an integrated nonlinear neural ODE.

Let us consider a spatial diffeomorphism $x^i(t)\mapsto \sum_j G^i{}_j(t) x^j(t)$. For the input-output relation to be unchanged, we require $G^i{}_j(0)=G^i{}_j(T)=\delta^i{}_j$. Then the transformation on the elements appearing in Eq.~\eqref{solxxx2} is equivalent to the rescaling symmetry for the self-attention layer, Eq.~\eqref{sagauge} with $\alpha=1$. The $\alpha$-scaling part of Eq.~\eqref{sagauge} is understood as an extra redundancy in the tensor decomposition \eqref{tensorlambda}, 
$\lambda^i{}_j \mapsto \alpha^{-1}\lambda^i{}_j$ and $\tilde{\lambda}_{kl} \mapsto \alpha \tilde{\lambda}_{kl}$.


\section{Regularization as a gauge fixing}
\label{sec6}

Gauge symmetries play an important role in learning dynamics: specifically, regularizations utilize the gauge symmetry to realize a model 
with desirable properties. 
In this section we see how physics intuition on gauge symmetries may be able to control the dynamics.
The gauge symmetry leaves the loss function intact, under a transformation of the weights.
Let us call the subspace of weights that can be reached by gauge transformations starting from a reference point $\theta$ as a {\it gauge orbit} represented by $\theta$.
The gradient of the loss function becomes orthogonal to the gauge orbits, directing training along a vector normal to these gauge orbits \cite{kunin2020neural}.
In other words, a gauge is chosen once a set of initial weights is picked.
Introducing a regularization that breaks the gauge symmetry accounts for a {\it gauge fixing condition} in physics, which results in weight updates also along the gauge orbits. Below we build some physics intuition for regularizations.

In neural ODEs, a regularization term motivated by the optimal transport problem has been introduced~\cite{finlay2020train},\footnote{It would be interesting to find a physical interpretation of the neural ODE regularizer using random time sampling \cite{ghosh2020steer}.} in which the optimal transport should bring $x(t)$ linear in $t$, {\it i.e.}~the motion is driven to be a linear uniform motion {\it a la} Galileo Galilei. Let us work out the condition of the linear uniform motion. Looking at the linear neural ODE \eqref{eq:dot-x-wx-b2}, we find that the condition
is equivalent to
\begin{align}
    \frac{d}{dt}
    \left[
    \sum_j w_i^{\,\, j}(t) x_j(t) + b_i(t)\right]=0.
\end{align}
Expanding this equation and using again \eqref{eq:dot-x-wx-b2}, we find that, for this to be valid for any data $x$, 
\begin{align}
     \dot{w}+w^2 = 0 , \quad
     (\partial_t +w)b=0. 
\end{align}
The regularizer with an $L^2$-norm is
\begin{align}
    L_{\rm R} = \lambda \int_0^T dt
    \left[
    (\dot{w}+w^2)^2
    +\right[(\partial_t +w)b\left]^2
    \right],
    \label{R}
\end{align}
where $\lambda$ is a hyperparameter for the regularizer strength.
This regularizer \eqref{R} allows a physics interpretation. 
As we have mentioned, $w$ and $b$ are interpreted as a gauge field and a Higgs field. 
The first term in Eq.~\eqref{R} is a so-called $R_\xi$-gauge term in gauge quantum field theories.
The second term is a Higgs kinetic term with a gauge covariant derivative.\footnote{This term is not gauge invariant, as the present gauge transformation is not of a phase rotation but the rescaling.
If we restrict the transformation parameter $G(t)$ to be just a spatial rotation (the special orthogonal group for the diffeomorphism transformation), we can write 
a gauge-invariant action for the Higgs kinetic term which 
respects the rotation symmetry but does not respect an arbitrary scaling symmetry. 
It works as a weak regularizer respecting the rotation symmetry.} With this physics 
intuition one can introduce different versions of regularizers which physicists explored
in formulating quantum field theories.

%

Supposing that we do not introduce any regularization to break the gauge symmetry, and train the neural ODE to reach a trained model.\footnote{As a side remark, we make a brief comment on conserved quantity. According to Noether's theorem \cite{Noether1918}, a symmetry with constant and continuous transformation parameters (called a global symmetry in physics) leads to the existence of a conserved current. In our case, due to the constraint \eqref{eq:epsilon-0} and \eqref{eq:epsilon-T} there is no such symmetry, thus we do not expect any conserved quantity at the moment. As for the quantity conserved during the training time, see \cite{kunin2020neural,tanaka2021noether}.} 
We may regard that this is a spontaneous symmetry breaking \cite{nambu1961dynamical1,nambu1961dynamical2}, since a certain nonzero Higgs field configuration $b(t)$ as an initial configuration or as a learning dynamics is used. In particular, the trained model corresponds to a vacuum in classical gauge field theory of the gauge field $w(t)$ and the Higgs field $b(t)$. Due to the gauge symmetry, there exist an infinite number of other vacua, which forms a continuous family, the vacuum gauge orbit. The trained model is a random sampling of the vacuum gauge orbit.\footnote{In physics, spontaneous symmetry breaking of gauge symmetries leads to Higgs phenomenon \cite{higgs1966spontaneous} where gauge bosons acquire a mass. In the present case the system is in 1 dimension (of time) and there is no physical degree of freedom for the gauge boson (no kinetic term can be written), thus resulting in no Higgs mechanism in neural ODEs.}

\section{Conclusions}
\label{sec7}

In this paper we have found that the gauge symmetry of neural ODEs is spacetime diffeomorphism, and use it for interpreting the rescaling symmetries present in general feedforward neural networks and transformers. This is possible through the integrated relation between the continuous neural ODEs and the discrete architecture of neural networks. Although a part of the diffeomorphism is often explicitly broken for concrete architectures and also with popular regularizations, the symmetry perspective of neural network bridges the machine learning and physics, which will help opening the black box of inner-working mechanism of neural networks.

Fundamental physics world, emanating from the gauge symmetry viewpoint, applies to diverse generalities of physical models. Supported by this physical diversity, one can attempt to construct novel machine learning model with physically controlled dynamics. At such an occasion, it is crucial to
secure the symmetry protected under the mapping between physical models (ODEs) and the excellent machine learning models such as transformers. Our findings contribute to securing the bridge, for future models.



\section{Impact statement}
\label{sec8}
This paper introduces a novel approach to understanding neural networks by unifying symmetries within them, particularly in Transformer models and neural ODEs, through the application of gauge symmetry concepts from physics. Our work provides a new perspective in comprehending the complex behaviors of deep learning models by integrating physical principles into machine learning. We believe this research contributes to advancing the field of Machine Learning and Physics, offering insights that could lead to innovative developments. While our study focuses on theoretical aspects, we acknowledge the potential societal implications of advancements in AI, and emphasize the importance of ethical considerations and responsible use of such technologies.

\section*{Acknowledgements}

K.~H.~would like to thank D.~Berman, S.~Sugishita and N.~Yokoi for valuable discussions, and String Data 2023 conference at Caltech for hospitality.
The work of K.~H.~was supported in part by JSPS KAKENHI Grant Nos.~JP22H01217, JP22H05111, and JP22H05115.
The work of Y.~H.~was supported in part by JSPS KAKENHI Grant No. JP22H05111. 
The work of A.~S.~was supported in part by JSPS KAKENHI Grant No. JP20K03743, JP23H04484 and JST PRESTO JPMJPR2123.

\nocite{langley00}

\bibliography{paper}
\bibliographystyle{icml2024}

\newpage
\appendix
\onecolumn

 \section{Ethics review}

This paper focuses on the theoretical aspects of integrating symmetries within neural networks, specifically Transformer models and neural ODEs. We believe our work does not raise specific ethical issues beyond the general scope of those applied to machine learning research. We are committed to the responsible development and use of AI technologies and adhere to ethical guidelines in our research. We acknowledge the importance of ongoing consideration of ethical implications as the field of AI continues to evolve.

\section{ Proof of Theorem \ref{thm:node-nn} }\label{app:proof-linear-nn}

We here give the proof of Theorem~\ref{thm:node-nn}. 

We first show that a ``Wilson line'' operator defined in Eq.~\eqref{eq:wilson-line} is transformed by a gauge transformation as
\begin{equation}
W(w')_{t_1:t_2} 
=
G^{-1}(t_1)
W(w)_{t_1:t_2} 
G(t_2). 
\label{eq:wilson-line-transfm}
\end{equation}
To show this, we look at the transformation property of infinitely short Wilson line,
$W(w)_{(t+\Delta):t} 
= 1 + w(t) \Delta$,
where $\Delta \in \mathbb R$ is infinitesimal.
Under a gauge transformation,
\begin{equation}
\begin{split}
W(w')_{(t+\Delta):t}    
&= 
1 + w'(t) \Delta
\\
&= 
1 + 
G^{-1}(t) w(t) G(t) \Delta - 
G^{-1}(t) \dot G (t) \Delta .
\end{split}
\end{equation}
Using 
$G^{-1}(t) \dot G (t) =- \dot{G}^{-1}(t) G (t)$,
we have 
\begin{equation}
W(w')_{(t+\Delta):t}
= 
G^{-1} (t+\Delta)
\left[ 1 + w(t) \Delta \right] G (t)
+ O(\Delta^2) . 
\label{eq:wilson-line-delta-tr} 
\end{equation}
The Wilson lines \eqref{eq:wilson-line-transfm} with finite time lengths 
can be decomposed into the product of 
$W(w)_{(t+\Delta):t}$, 
and the transformation property \eqref{eq:wilson-line-transfm}
follows from Eq.~\eqref{eq:wilson-line-delta-tr}.

The parameters of the discrete linear neural network is expressed 
in terms of those of ${\rm NODE}_{A}$ as 
Eq.~\eqref{eq:params-nn-node}.
Using the relation~\eqref{eq:wilson-line-transfm}, 
one can easily show that a diffeomorphism of the linear neural ODE
parametrized by 
$\{G(t), c(t) \}_{t \in [0,T]}$ 
induces a transformation given by Eq.~\eqref{eq:lin-nn-gauge}
of the discrete linear neural network 
with $G_n = G(n \Delta)$ and 
$c_n = c(n \Delta)$.

By assumption, there is a set of 
transformation parameters 
$\{G_n, c_n \}_{n=0,\ldots,N-1}$ 
that connects discrete neural network  
${\rm NN}_A$ and ${\rm NN}_B$. 
We can choose the parameters 
$\{G(t), c(t) \}_{t \in [0,T]}$ 
for a diffeomorphism 
such that $G_n = G(n \Delta)$
and $c_n = c(n \Delta)$. 
Applying a diffeomorphism with these parameters to ${\rm NODE}_A$, 
we obtain ${\rm NODE}_B$ of desired properties.
%
%
This concludes the proof.

\section{Rescaling symmetry in convolutional neural networks}
\label{appB}

In this appendix, we study the rescaling symmetries in convolutional neural networks and graph neural networks.

A generic convolutional layer is defined as
\begin{align}
    &x_{ij}^{(n+1)} = \varphi(u_{ij}^{(n)}), \\
    &u_{ij}^{(n)} = \sum_{p=0}^{f-1}\sum_{q=0}^{f-1} x_{i+p, j+q}^{(n)} h_{pq}^{(n)}.
\end{align}
Here $h$ is a filter of the size $f \times f$. In generic convolutional neural networks, one uses a pooling layer. The $L_s$ pooling layer is
\begin{align}
    y_{ij}^{(n+1)}
     = \left[
     \frac{1}{{\rm dim}(P)}\sum_{(p,q)\in P} (y^{(n)}_{pq})^s
     \right]^{1/s},
\end{align}
and $s=1$ corresponds to the average pooling, while $s=\infty$ corresponds to the max pooling.

It is easy to find that the following scaling symmetry exists in the successive convolutional layers,
\begin{align}
\begin{split}
    & h_{p,q}^{(n)} \mapsto \alpha^{(n)} h_{p,q}^{(n)}
    \\
    &    h_{p,q}^{(n+1)} \mapsto [\alpha^{(n)}]^{-1} h_{p,q}^{(n+1)}
    \end{split}
    \label{cnngauge}
\end{align}
with $\alpha^{(n)} \in {\mathbb R}_{>0}$.
The scaling symmetry survives 
with any insertion of the $L_s$ pooling layer with any $s$ in the set of the convolutional layers.

The dimensions of the rescaling symmetry is greatly reduced compared to that
of the fully connected deep neural network \eqref{gsdnn}, since the gauge transformation parameter $\alpha^{(n)}$ is shared for any unit in a given convolutional layer. This is due to the translation symmetry imposed on the convolutional layer, as the filter $h$ needs to be shared for any grid of the input.

One additional comment is that once a normalization layer is inserted the scaling rescaling symmetry is ruined, as the normalization process kills the scaling.

We immediately see that the rescaling symmetry in the convolutional neural networks \eqref{cnngauge} should be also there for graph convolutional network, as the latter is regarded as a sparse version of the convolutional neural networks.

\end{document}